\documentclass{article}
\usepackage{graphicx} %

\usepackage{amsmath}
\usepackage{amsthm}
\usepackage{amssymb}
\usepackage{color}
\usepackage{mathtools}
\usepackage{makecell}
\usepackage{bm}
\usepackage{todonotes}
\usepackage{diagbox}
\usepackage{tablefootnote}
\usepackage{hyperref}
\usepackage[ruled,vlined,linesnumbered]{algorithm2e}
\usepackage{layouts}
\DeclarePairedDelimiter{\ceil}{\lceil}{\rceil}
\DeclarePairedDelimiter{\floor}{\lfloor}{\rfloor}

\DeclareMathOperator{\rank}{rank}

\let\baraccent=\= %
\renewcommand{\=}[1]{\stackrel{#1}{=}} %

\providecommand{\cX}{\mathcal{X}}
\providecommand{\cY}{\mathcal{Y}}

\mathchardef\mhyphen="2D %

\newcommand{\interior}[1]{%
  {\kern0pt#1}^{\mathrm{o}}%
}

\usepackage[capitalise]{cleveref}
\newtheorem{theorem}{Theorem}[section]
\newtheorem{lemma}[theorem]{Lemma}

\newtheorem{proposition}[theorem]{Proposition}
\newtheorem{corollary}[theorem]{Corollary}

\newtheorem{definition}[theorem]{Definition}

\usepackage[utf8]{inputenc} %
\usepackage[T1]{fontenc}    %
\usepackage{hyperref}       %
\usepackage{url}            %
\usepackage{booktabs,comment}       %
\usepackage{amsfonts}       %
\usepackage{nicefrac}       %
\usepackage{microtype}      %
\usepackage{xcolor}         %
\usepackage{amsmath}
\usepackage{amssymb}
\usepackage{amsthm}
\usepackage{bbm}
\usepackage{xcolor}
\usepackage{graphicx}
\usepackage{subcaption}
\usepackage{adjustbox}
\usepackage{url}

\usepackage{hyperref}
\hypersetup{
    colorlinks,
    linkcolor={blue!80!black},
    citecolor={green!50!black},
}
\usepackage{xcolor}
\colorlet{linkequation}{blue}

\usepackage{comment}
\usepackage{latexsym}
\usepackage{amsmath}
\usepackage{amssymb}
\usepackage{amsthm}
\usepackage{epsfig}
\usepackage{graphicx}
\usepackage{bm}
\usepackage[shortlabels]{enumitem}
\usepackage{graphicx}
\usepackage{bbm}
\usepackage{accents}
\usepackage{mathtools}

\usepackage[top=1in, bottom=1in, left=1in, right=1in]{geometry}

\newcommand{\R}{\mathbb{R}}

\newcommand{\id}{\mathrm{id}}

\newcommand{\diag}{\text{diag}}

\DeclareMathOperator*{\argmin}{arg\,min}

\newtheoremstyle{myremark} %
    {\topsep}                    %
    {\topsep}                    %
    {\rm}                        %
    {}                           %
    {\bf}                        %
    {.}                          %
    {.5em}                       %
    {}  %

\DeclareSymbolFont{rsfs}{U}{rsfs}{m}{n}
\DeclareSymbolFontAlphabet{\mathscrsfs}{rsfs}

\def\balpha{{\boldsymbol \alpha}}

\def\cX{{\mathcal X}}

\def\cX{{\mathcal X}}

\def\diag{{\rm diag}}

\def\bsigma{{\boldsymbol \sigma}}

\crefname{claim}{claim}{claims}
\crefname{fact}{fact}{facts}

\usepackage{hyperref}
\usepackage{url}

\usepackage[many]{tcolorbox}

\usepackage{varwidth}%

\usepackage{hyperref}
\hypersetup{
    colorlinks,
    linkcolor={blue!80!black},
    citecolor={green!50!black},
}

\usepackage[top=1in, bottom=1in, left=1in, right=1in]{geometry}

\usepackage{amsmath}
\numberwithin{equation}{section}
\usepackage{amssymb}
\usepackage{mathtools}
\usepackage{amsthm}

\usepackage{thmtools}
\usepackage{thm-restate}

\usepackage[numbers, compress]{natbib}

\usepackage{tikz}
\usetikzlibrary{shapes, arrows, calc, positioning, arrows.meta}

\usepackage{wrapfig}
\usepackage{sidecap}

\usepackage{fancybox}
\newenvironment{fminipage}%
  {\begin{Sbox}\begin{minipage}}%
  {\end{minipage}\end{Sbox}\fbox{\TheSbox}}

\usepackage{longtable}
\usepackage{booktabs}
\usepackage{soul}
\setuldepth{hi}

\title{On the inductive bias of infinite-depth ResNets and the bottleneck rank}

\author{Enric Boix-Adsera\thanks{\texttt{eboix@mit.edu}, MIT Mathematics and Harvard CMSA}}
\date{\today}

\begin{document}

\maketitle

\begin{abstract}
We compute the minimum-norm weights of a deep \textit{linear} ResNet, and find that the inductive bias of this architecture lies between minimizing nuclear norm and rank. This implies that, with appropriate hyperparameters, deep \textit{nonlinear} ResNets have an inductive bias towards minimizing bottleneck rank.
\end{abstract}

\section{Introduction}\label{sec:introduction}

There are three key elements that affect the performance of a trained neural network: the data, the optimizer, and the architecture. This work focuses on the architecture—specifically, the deep ResNet architecture \cite{he2016deep}, which is a core component of modern neural networks, including transformers \cite{vaswani2017attention}. We analyze its inductive bias, which we define as follows. Given a neural network architecture $f_{\theta} : \cX \to \cY$ parametrized by weights $\theta$, the cost of representing a function $f : \cX \to \cY$ is the minimum norm of the weights required to represent it:
\begin{align*}
\mathrm{cost}(f) = \min \{\|\theta\|_F^2 \mbox{ such that } f_{\theta} \equiv f\}\,.
\end{align*}
When the network is trained with weight decay or $L_2$ weight regularization, the training procedure is biased towards outputting functions with lower cost.\footnote{We note that $\mathrm{cost}(f)$ is an approximation to the inductive bias. It does not fully determine the inductive bias, since the optimization dynamics also play a role and the minimum-cost solutions might not be found. Additionally, when training occurs without weight decay or $L_2$ weight regularization, the inductive bias may differ. See, e.g., \cite{arora2019implicit,li2020towards,razin2020implicit}.} This inductive bias is determined by the neural network parametrization $f_{\theta}$ (i.e., its architecture). For instance, deep linear networks favor learning low-rank linear transformations \cite{gunasekar2017implicit,arora2019implicit}, linear convolutional networks favor Fourier-sparse solutions \cite{gunasekar2018implicit}, and diagonal linear networks favor sparse solutions \cite{gunasekar2018implicit}.

The result of \cite{jacot2022implicit} on the inductive bias of deep nonlinear networks is particularly relevant to this paper, and forms the basis of our investigations. \cite{jacot2022implicit} defines the ``bottleneck rank'', a notion of rank 
for nonlinear functions. Roughly speaking, the bottleneck rank of a function $f : \R^{n_1} \to \R^{n_2}$ is the smallest dimension $k$ such that can be written as $f = g \circ h$ for two functions $g : \R^k \to \R^{n_2}$ and $h : \R^{n_1} \to \R^k$ (where $g,h$ satisfy regularity conditions). The papers \cite{jacot2022implicit,jacot2023bottleneck} show that deep networks with nonlinear fully-connected layers have an inductive bias towards learning functions with a low bottleneck rank. This result has been extended to leaky-ResNets \cite{jacot2024hamiltonian} and to convolutional networks \cite{wen2024frequencies} (but the latter requires a modified notion of bottleneck rank).

The arguments for the bottleneck rank are based on the observation that the minimum-cost network representing function $f$ has the first layers represent $h$, the last layers represent $g$, and the vast majority of the layers in the middle represent the identity map on a subspace of dimension $k$. When composed together, these layers yield the function $f$. Since most layers are dedicated to representing the identity map, this dominates the cost when the architecture does not have skip connections or only has leaky skip connections. It follows that fully-connected networks and leaky-ResNets have an inductive bias for low bottleneck rank.

In contrast, ResNets explicitly incorporate skip connections that trivialize representing the identity function, raising the question of whether the bottleneck rank framework is still relevant to ResNet-based models which are used in practice:

\begin{center}
\textit{What is the inductive bias of ResNets? Is it related to the bottleneck rank?}
\end{center}

Naively, the answer to this question appears to be that ResNets do not have a bias towards low bottleneck rank. Indeed, consider a neural network that consists of blocks with skip connections.\footnote{Such a model is equivalent in its infinite-depth limit to a NeuralODE, as studied in, e.g., \cite{chen2018neural,owhadi2020ideas,geshkovski2023mathematical}.} Such a network can trivially represent the full-rank identity function by just setting all of the blocks to zero. This appears to invalidate the bottleneck rank theory for ResNets since the identity function has very high bottleneck rank.

However, the above argument neglects an important part of the ResNet architecture. Real-world ResNets also include embedding and unembedding layers (i.e., linear input and output transforms). Perhaps surprisingly, we prove that because of the presence of these embedding and unembedding weight matrices, ResNets have a preference for low-rank solutions even though they have skip connections.

In Section~\ref{sec:linear}, we explicitly characterize the inductive bias of deep \textit{linear} ResNet architectures, proving that they have a bias towards finding low-rank linear transformations. In Section~\ref{sec:nonlinear}, we show that this additionally implies that, in certain hyperparameter regimes, deep \textit{nonlinear} ResNet architectures have a preference towards solutions with low bottleneck rank. 

From a practical standpoint, these results suggest that while skip connections do simplify representing identity transformations, the  bottleneck rank continues to offer insights about how ResNets generalize, and remains a valuable tool to understand the inductive bias of neural networks.

\section{The inductive bias of linear residual networks lies between minimizing nuclear norm and minimizing rank}\label{sec:linear}

In this section, we consider \textit{linear} residual networks with $L$ layers, as well as embedding and unembedding matrices $W_e \in \R^{n \times d_{in}}$ and $W_u \in \R^{d_{out} \times n}$. These correspond to residual networks with MLPs at each layer, with identity activation function. Because the activation function is linear, these networks can only represent linear functions. For simplicity, we will suppose $n \geq \max(d_{in},d_{out})$.

\paragraph{Architecture} We study two variants of linear residual networks. In the first variant, the network has depth-1 residual blocks $W_1,\ldots,W_L \in \R^{n \times n}$:
$$f_{1-lin}(W_u,W_e,\{W_i\}_{i \in [L]}) = W_u(I+W_1)(I+W_2) \dots (I+W_L)W_e \in \R^{d_{out} \times d_{in}}.$$
In the second variant, the network has depth-2 residual blocks $W_{1,1},W_{1,2},W_{2,1},W_{2,2},\ldots,W_{L,1},W_{L,2} \in \R^{n \times n}$, as is the case in the MLP layers in transformer architectures:
$$f_{2-lin}(W_u,W_e, \{W_{i,j}\}_{i \in [L], j \in [2]}) = W_u(I+W_{1,2}W_{1,1})(I+W_{2,2}W_{2,1}) \dots (I+W_{L,2}W_{L,1})W_e \in \R^{d_{out} \times d_{in}}.$$

\paragraph{Cost of network} We define the cost for a certain choice of parameters to be the sum of the squared Frobenius norms of the weights, since this is what is penalized when training with weight-decay or explicit $L_2$ regularization. For any matrix $A \in \R^{n \times n}$ and a parameter $\lambda \in (0,\infty)$ we define the cost of representing a linear transformation $A \in \R^{d_{out} \times d_{in}}$ with a network with depth-1 blocks by:
\begin{align*}
c^{1-lin}_{L,n,\lambda}(A) &= \min_{\substack{W_u,W_e,\{W_i\}_{i \in [L]} \\ f_{1-lin}(W_u,W_e,\{W_i\}) = A}} \frac{1}{2}\|W_u\|_F^2 + \frac{1}{2}\|W_e\|_F^2 + \lambda L\sum_{i \in [L]} \|W_i\|_F^2\,,
\end{align*}
and by a network with depth-2 blocks by:
\begin{align*}
c^{2-lin}_{L,n,\lambda}(A) &= \min_{\substack{W_u,W_e,\{W_{i,j}\}_{i \in [L], j \in [2]} \\ f_{2-lin}(W_u,W_e,\{W_{i,j}\}) = A}} \frac{1}{2}\|W_u\|_F^2 + \frac{1}{2}\|W_e\|_F^2 + \frac{\lambda}{2}\sum_{i \in [L], j \in [2]} \|W_{i,j}\|_F^2\,.
\end{align*}
The parameter $\lambda \in (0,\infty)$ allows us to interpolate between weighting the cost of the embedding and unembedding layers, versus the cost of the residual layers in the network.\footnote{The scaling with $L$ is chosen so that the limiting cost as $L \to \infty$ is nontrivial. Namely, for any $\lambda \in (0,\infty)$, there is a constant amount of cost on the embedding/unembedding weights, and also on the residual weights.} 

Our main result in this section is the following explicit formula for the costs of the minimum-norm neural networks. The key finding is that the cost decomposes additively across the singular values of the linear transformation. 
\begin{theorem}[Explicit formula for cost in linear case]\label{thm:main-lin-k1-k2}
For any linear transformation $A \in \R^{d_{out} \times d_{in}}$, any parameter $\lambda \in (0,\infty)$, any depth $L \geq 1$, and any width $n \geq \rank(A)$, we have
\begin{align}
c^{1-lin}_{L,n,\lambda}(A) &= \sum_{i=1}^{n} c^{1-lin}_{L,1,\lambda}([\sigma_i(A)]) = \sum_{i=1}^n \left(\min_{\alpha \in \R_{\geq 0}} \frac{\sigma_i(A)}{(1+\alpha/L)^L}  + \lambda\alpha^2\right) \label{eq:cost-lin-1} \\
c^{2-lin}_{L,n,\lambda}(A) &= \sum_{i=1}^n c^{2-lin}_{L,1,\lambda}([\sigma_i(A)]) = \sum_{i : \sigma_i(A) \leq \lambda} \sigma_i(A) + \sum_{i : \sigma_i(A) > \lambda} \lambda \left((L+1)(\sigma_i(A)/\lambda)^{1/(L+1)}-L\right)\,. \label{eq:cost-lin-2}\end{align}
\end{theorem}

The proof is deferred to Section~\ref{sec:proof}. As a corollary of this result, we derive the cost in the limit of infinite depth.

\begin{corollary}[Cost of infinite-depth linear residual network]\label{cor:lin-infinite-depth-k1-k2} Under the conditions of Theorem~\ref{thm:main-lin-k1-k2},
\begin{align*}
c_{n,\lambda}^{1-lin}(A) &:= \lim_{L \to \infty} c_{L,n,\lambda}^{1-lin}(A) = \sum_{i=1}^n \left(\min_{\alpha \in \R_{\geq 0}} \frac{\sigma_i(A)}{\exp(\alpha)} + \lambda \alpha^2\right)\\
c_{n,\lambda}^{2-lin}(A) &:= \lim_{L \to \infty} c_{L,n,\lambda}^{2-lin}(A) =  \sum_{i : \sigma_i(A) \leq \lambda} \sigma_i(A) + \sum_{i : \sigma_i(A) > \lambda} \lambda(1 + \log(\sigma_i(A)/\lambda))\,.
\end{align*}
\end{corollary}
\begin{proof}
For the first equality, let $\sigma \geq 0$ and $\lambda > 0$ and for any $L$ let $\alpha_L \in \argmin_{\alpha \in \R_{\geq 0}} \sigma/(1+\alpha/L)^L + \frac{\lambda}{2} \alpha^2$. One can see that $\alpha_L \leq C_{\sigma,\lambda}$ for a constant $C_{\sigma,\lambda}$. Therefore $\lim_{L \to \infty} |\frac{\sigma}{\exp(\alpha_L)} - \frac{\sigma}{(1+\alpha_L/L)^L}| = 0$. 

For $c_{n,\lambda}^{2-lin}(A)$, this is a direct calculation, using that $(L+1)((\sigma/\lambda)^{1/(L+1)} - 1) \to \log(\sigma/\lambda)$.
\end{proof}

These expressions are still fairly complicated, so let us analyze the asymptotics of these costs, as either $\lambda \to 0$ and the embedding/unembedding becomes the dominating term in the cost, or as $\lambda \to \infty$ and the residual part of the network becomes the dominating term in the cost.
\begin{corollary}[Cost of infinite-depth residual network interpolates between nuclear norm and rank]\label{cor:linear-networks-lambda-asymptotics} Let $A \in \R^{d_{out} \times d_{in}}$, and $n \geq \rank(A)$.
\begin{itemize} 
\item \textbf{Nuclear norm minimization for large $\lambda$}. Taking $\lambda \to \infty$ recovers the nuclear norm $\|A\|_*$.
\begin{align}\label{eq:lin-asymp-infty}
\lim_{\lambda \to \infty} c_{n,\lambda}^{1-lin}(A) = \lim_{\lambda \to \infty} c_{n,\lambda}^{2-lin}(A) = \|A\|_*
\end{align}

\item \textbf{Rank minimization for small $\lambda$}. Taking $\lambda \to 0$ leads to a cost that is, to first order, the rank of $A$.

\begin{align}\label{eq:lin-asymp-0}
\lim_{\lambda \to 0}
\frac{c^{1-lin}_{L,n,\lambda}(A)}{\lambda (\log(1/\lambda))^2} = \lim_{\lambda \to 0} \frac{c^{2-lin}_{n,\lambda}(A)}{\lambda \log(1/\lambda)} = \rank(A)
\end{align}

\end{itemize}

\end{corollary}
\begin{proof}
First consider the limit $\lambda \to \infty$. For any $\sigma$, we have
$\lim_{\lambda \to \infty}(\min_{\alpha \in \R_{\geq 0}} \sigma/\exp(\alpha) + \lambda \alpha^2) = \sigma$, proving the formula for depth-1 blocks. For depth-2 blocks, we have $c_{n,\lambda}^{2-lin}(A) = \sum_{i} \sigma_i(A) = \|A\|_*$ for large enough $\lambda$. 

Now consider the limit $\lambda \to 0$. For any $\sigma$, we have 
\begin{align*}
\lim_{\lambda \to 0} \frac{\left(\min_{\alpha \in \R_{\geq 0}} \frac{\sigma}{\exp(\alpha)} + \lambda \alpha^2 \right)}{\lambda (\log(1/\lambda))^2}  = \lim_{\lambda \to 0} \frac{\left(\min_{\beta \in \R_{\geq 0}} \frac{\lambda \sigma}{\exp(\beta)} + \lambda (\beta + \log(1/\lambda))^2 \right)}{\lambda (\log(1/\lambda))^2}  = \begin{cases} 1, & \mbox{ if } \sigma \neq 0 \\ 0, & \mbox{ if } \sigma = 0 \end{cases}\,,
\end{align*}
since if $\sigma = 0$ we can take $\alpha = 0$. And $\sigma \neq 0$, one can show that the optimal $\beta$ must satisfy $|\beta| \leq\sqrt{\log(1/\lambda)}$ for small enough $\lambda$. Combined with Corollary~\ref{cor:lin-infinite-depth-k1-k2}, and a direct calculation in the case of depth-2 blocks, this implies \eqref{eq:lin-asymp-0}.
\end{proof}

Therefore, finding the smallest-cost linear residual network that fits the data corresponds to finding the minimum rank linear transformation that fits the data (for very small $\lambda \ll 1$), and  to the minimum nuclear-norm linear transformation (for very large $\lambda \gg 1$).

\subsection{Proof of Theorem~\ref{thm:main-lin-k1-k2}}\label{sec:proof}

Each of the equalities in Theorem~\ref{thm:main-lin-k1-k2} has two parts: an upper bound, and a lower bound. The ``easier'' direction is the upper bound, which we will show via a direct construction of a neural network with the correct cost. The ``harder'' direction is the lower bound, and here the essential ingredient is the following inequality of \cite{gel1950relation} concerning the singular values of products of matrices.
\begin{proposition}[Proved in \cite{gel1950relation}; see Corollary 2.4 of \cite{li1999lidskii} or Theorem III.4.5 in \cite{bhatia1996matrix}]\label{prop:mult-lidskii-mirsky-wielandt-main}
For any matrices $A \in \R^{n_0 \times n}, C \in \R^{n \times n_1},B \in \R^{n \times n}$ and any $k \in \{1,\ldots,n\}$ and any indices $1 \leq i_1 < \dots i_k \leq n$ such that $\sigma_{i_j}(B) \neq 0$, we have
\begin{align*}
\prod_{j=1}^k \frac{\sigma_{i_j}(ABC)}{\sigma_{i_j}(B)} \leq \prod_{j=1}^k \sigma_j(A)\sigma_j(C)\,.
\end{align*}
\end{proposition}

This, in turn, implies the following lemma.
\begin{lemma}\label{lem:prod-sing-values-convex-decomp}
For any $A \in \R^{n_1 \times n_2}, B \in \R^{n_2 \times n_3}$, and convex non-decreasing $g : [0,\infty) \to \R$, we have
\begin{align*}
\sum_{i=1}^{\rank(AB)} g(\sigma_i(A)) \geq \sum_{i=1}^{\rank(AB)} g(\sigma_i(AB)/\sigma_i(B))\,
\end{align*}
\end{lemma}
\begin{proof}
Let $k = \rank(AB)$, and for any $i \in [k]$ define $r_i = \sigma_i(B)/\sigma_i(AB)$. Order these ratios as $r_{[1]} \geq r_{[2]} \geq \dots \geq r_{[k]}$. By Proposition~\ref{prop:mult-lidskii-mirsky-wielandt-main}, for any $1 \leq j \leq \rank(AB)$, we have
$\sum_{i=1}^j \log(\sigma_i(A)) \geq \sum_{i=1}^j \log(r_{[i]})$. In other words, $\log(\sigma_i(A))$ weakly submajorizes $\log(r_{[i]})$. Since the function $t \mapsto g(\exp(t))$ is convex and non-decreasing, by fact 3.C.1.b in \cite{marshall2011inequalities}, we have $\sum_{i=1}^k g(\sigma_i(AB)) \geq \sum_{i=1}^k f(r_{[i]}) = \sum_{i=1}^k g(\sigma_i(AB)/\sigma_i(B))$.
\end{proof}

This result allows us to determine the minimum cost of networks with depth-1 blocks, and any increasing, convex costs on the weight matrices in the residual layers.
\begin{lemma}\label{lem:main-lin-convex-k1}
For any $A \in \R^{d_{in} \times d_{out}}$, any depth $L$, width $n \geq \rank(A)$, and convex, increasing $f : [0,\infty) \to \R$ with $f(0) = 0$, define
\begin{align*}
c_{L,n}^{1-lin}(A;f) := \min_{\substack{W_u, W_e, \{W_i\}_{i \in [L]} \\ f_{1-lin}(W_u,W_e,\{W_i\}) = A}} \frac{1}{2} \|W_u\|_F^2 + \frac{1}{2} \|W_e\|_F^2 + \sum_{i\in [L]} \sum_{j \in [n]} f(\sigma_j(W_i))\,.
\end{align*}
Then the cost decomposes additively across the singular values:
\begin{align*}
c_{L,n}^{1-lin}(A;f) := \sum_{i=1}^{\rank(A)} c_{L,1}^{1-lin}([\sigma_i(A)];f) = \sum_{i=1}^{\rank(A)} \left(\min_{\alpha \in \R_{\geq 0}} \frac{\sigma_i(A)}{(1+\alpha)^L} + Lf(\alpha)\right)\,.
\end{align*}
\end{lemma}
\begin{proof}

\textbf{Upper bound}. For the upper bound, let $r = \min(d_{out}, d_{in})$, and write the SVD decomposition $A = U \diag(\bsigma(A)) V^{\top}$ where $U \in \R^{d_{out} \times r}$, $V \in \R^{r \times d_{in}}$ are semi-orthogonal matrices. Then, for some $\balpha \in \R^{r}_{ \geq 0}$, let 
\begin{align*}
W_u = U \begin{bmatrix} \diag(\sqrt{\frac{\bsigma(A)}{(1+\balpha)^L}}) & 0\end{bmatrix} \in \R^{d_{out} \times n}, \  
W_e = \begin{bmatrix} \diag(\sqrt{\frac{\bsigma(A)}{(1+\balpha)^L}}) \\ 0 \end{bmatrix} V^{\top} \in \R^{n \times d_{in}}, \ W_i = \begin{bmatrix} \diag(\balpha) & 0 \\ 0 & 0 \end{bmatrix} \in \R^{n \times n}\,.
\end{align*}
By construction, for any choice of $\balpha$ we have $f_{1-lin}(W_u,W_e,\{W_i\}) = A$, and therefore
\begin{align*}
c_{L,n}^{1-lin}(A;f) \leq \min_{\balpha \in \R^{r}_{\geq 0}} \sum_{i=1}^{r} 
\left(\min_{\alpha \in \R_{\geq 0}} \frac{\sigma_i(A)}{(1+\alpha)^L} + Lf(\alpha)\right) = \sum_{i=1}^{\rank(A)} \left(\frac{\sigma_i(A)}{(1+\alpha)^L} + Lf(\alpha)\right)\,.
\end{align*}
\textbf{Lower bound}. It remains to prove the lower bound, which is the ``harder'' direction of this lemma. Suppose that $W_u,W_e,\{W_i\}$ achieve the minimum cost subject to $f_{1-lin}(W_u,W_e,\{W_i\}) = A$. For convenience, define the partial products $V_{\ell} = (I+W_1)\dots(I+W_{\ell})$, where $V_0 = I$, and also let $k = \rank(A)$. By two applications of
Lemma~\ref{lem:prod-sing-values-convex-decomp} with $g(t) = \frac{1}{2}t^2$, we can lower-bound the cost of the embedding and unembedding layers by
\begin{align}\label{ineq:lower-bound-embedding}
\frac{1}{2}(\|W_u\|_F^2 + \|W_e\|_F^2) &\geq\sum_{j=1}^{k} \frac{1}{2} \left(\frac{\sigma_j(W_uV_LW_e)}{\sigma_j(W_uV_L)}\right)^2 + \frac{1}{2} \left(\frac{\sigma_j(W_uV_L)}{\sigma_j(V_L)}\right)^2\,.
\end{align}
Additionally, we can lower-bound the cost of the residual layers. Define $\tilde{f}(t) = f(\max(t-1,0))$ and note that $\tilde{f}$ is convex and non-decreasing. By (a) using the fact $\sigma_j(W_i) \geq \max(\sigma_j(I + W_i)-1,0)$, and (b) applying Lemma~\ref{lem:prod-sing-values-convex-decomp} with $g(t) = \tilde{f}(t)$, we have
\begin{align}\label{ineq:lower-bound-embedding-2}
\sum_{j=1}^{n} f(\sigma_j(W_i)) &\stackrel{(a)}{\geq} \sum_{j=1}^{n} \tilde{f}(\sigma_j(I+W_i)) \stackrel{(b)}{\geq} \sum_{j=1}^{k} \tilde{f}(\sigma_j(V_i)/\sigma_j(V_{i-1}))\,.
\end{align}
Combining \eqref{ineq:lower-bound-embedding} and \eqref{ineq:lower-bound-embedding-2}, and using that $\sigma_j(W_uV_LW_e) = \sigma_j(A)$, and that $\sigma_j(V_0) = 1$, we have
\begin{align}
c_{L,n}^{1-lin}(A;f) &\geq \sum_{j=1}^{k} \frac{1}{2} \left(\frac{\sigma_j(W_uV_LW_e)}{\sigma_j(W_uV_L)}\right)^2 + \frac{1}{2} \left(\frac{\sigma_j(W_uV_L)}{\sigma_j(V_L)}\right)^2 + \sum_{i=1}^L \tilde{f}(\sigma_j(V_i)/\sigma_j(V_{i-1})) \nonumber \\
&\geq \sum_{j=1}^k \left(\inf_{u,v_1,\ldots,v_L \in \R_{> 0}} \frac{1}{2}((\sigma_j(A)/u)^2+(u/v_L)^2) + \tilde{f}(v_1) + \sum_{i=2}^L \tilde{f}(v_i / v_{i-1})\right) \nonumber \\
&= \sum_{j=1}^k \left(\inf_{\substack{u_1,u_2,v_1,\ldots,v_L \in \R_{> 0} \\ u_1u_2v_1v_2\dots v_L = \sigma_j(A)}} \frac{1}{2}((u_1)^2 + (u_2)^2) + \sum_{i=1}^L \tilde{f}(v_i)\right) \nonumber \\
&= \sum_{j=1}^k \left(\inf_{\substack{\beta,\alpha \in \R_{> 0} \\ \beta \alpha^L = \sigma_j(A)}} \beta + L \tilde{f}(\alpha)\right)\,, \label{ineq:c1-lin-lower-bound-intermediate}
\end{align}
where in the last line we use the convexity and increasing nature of the functions $t \mapsto \frac{1}{2}t^2$ and $t \mapsto \tilde{f}(t)$. Since $\tilde{f}$ is non-increasing and $f(t) = 0$ for $t \leq 1$, the optimum in \eqref{ineq:c1-lin-lower-bound-intermediate} is achieved at some $\alpha \geq 1$, which means
\begin{align*}
\eqref{ineq:c1-lin-lower-bound-intermediate} &= \sum_{j=1}^k \left(\min_{\alpha \in \R_{> 0}} \frac{\sigma_j(A)}{(1+\alpha)^L} + L\tilde{f}(\alpha-1)\right) = \sum_{j=1}^k \left(\min_{\alpha \in \R_{> 0}} \frac{\sigma_j(A)}{(1+\alpha)^L} + Lf(\alpha)\right)\,,
\end{align*}
as claimed.

\end{proof}

The main theorem follows directly as a special case of this lemma.
\begin{proof}[Proof of Theorem~\ref{thm:main-lin-k1-k2}]

For networks with depth-1 blocks, the equation \eqref{eq:cost-lin-1} follows immediately from Lemma~\ref{lem:main-lin-convex-k1} and $f(t) = \lambda L t^2$, since $\lambda L \|W_i\|_F^2 = \sum_{j=1}^n f(\sigma_j(W_i))$.

For networks with depth-2 blocks, notice that each of the depth-2 blocks $(W_{i,1},W_{i,2})$ can be viewed as a depth-1 block computing $M_i = W_{i,2}W_{i,1} \in \R^{n \times n}$ with cost equal to the nuclear norm $\|M_i\|_*$. This is because for any $M_i$, we have $\|M_i\|_* = \min \{\frac{1}{2} (\|W_{i,2}\|_F^2 + \|W_{i,1}\|_F^2) : W_{i,2}W_{i,1} = M_i\}$. In particular, $c_{L,n,\lambda}^{2-lin}(A) = c_{L,n}^{1-lin}(A;f)$ for the function $f(t) = \lambda t$, since $\lambda \|W_i\|_* = \sum_{j=1}^n f(\sigma_j(W_i))$. This implies
\begin{align*}
c_{L,n,\lambda}^{2-lin}(A) = \sum_{i=1}^{\rank(A)} \left(\min_{\alpha \in \R_{\geq 0}} \frac{\sigma_i(A)}{(1+\alpha)^L} + L\lambda \alpha\right)\,,
\end{align*}
which has the explicit solution reported in \eqref{eq:cost-lin-2}.

\end{proof}

\section{Nonlinear residual networks and the bottleneck rank}\label{sec:nonlinear}

We now consider \textit{nonlinear} residual networks with ReLU activations $\sigma(t) = \max(0,t)$.
These architectures represent functions $f : \R^{d_{in}} \to \R^{d_{out}}$, and are parametrized by $\theta$. We again study both depth-1 and depth-2 block networks, since they are the most popular variants.

For depth-1 block residual networks, the network $f_{1-nonlin}(\cdot;\theta)$ is parametrized by embedding/unembedding weights and biases are $W_u \in \R^{d_{out} \times n}, W_e \in \R^{n \times d_{in}}, b_u \in \R^n, b_e \in \R^{d_{out}}$, and the residual weights and biases are $W_i \in \R^{n \times n}, b_i \in \R^{n}$. The network is given by
\begin{align*}
f_{1-nonlin}(x;\theta) &= (f_u \circ (\id + f_L) \circ (\id + f_{L-1}) \circ \dots \circ (\id + f_1) \circ f_e)(x)\,, \\
& \mbox{ with unembedding } f_u(z;\theta) = b_u + W_uz, \mbox{ embedding } f_e(z;\theta) = b_e + W_ez, \\
& \mbox{ and internal layers }  f_{\ell}(z;\theta) = \sigma(W_{\ell} z + b_{\ell}).
\end{align*}

Depth-2 block residual networks are the same, except that the residual weights and biases are parametrized by $W_{i,j} \in \R^{n \times n}, b_{i,j} \in \R^{n}$, and each of the internal layers is a depth-2 network:
\begin{align*}
f_{2-nonlin}(x;\theta) &= (f_u \circ (\id + f_L) \circ (\id + f_{L-1}) \circ \dots \circ (\id + f_1) \circ f_e)(x)\,, \\
& \mbox{ with unembedding } f_u(z;\theta) = b_u + W_uz, \mbox{ embedding } f_e(z;\theta) = b_e + W_ez, \\
& \mbox{ and internal layers }  f_{\ell}(z;\theta) = W_{\ell,2}\sigma(W_{\ell,1} z + b_{\ell,1}) + b_{\ell,2}.
\end{align*}

\paragraph{Cost of representing FPLFs} Given a finite piecewise linear function (FPLF) $g : \Omega \to \R^{n_1}$ on a bounded set 
$\Omega \subseteq \R^{n_0}$, we define the minimum cost of representing this function with depth-1 blocks:
\begin{align*}
c^{1-nonlin}_{L,n,\lambda}(g; \Omega) = \inf_{\theta} \{ \frac{1}{2}\|W_u\|_F^2 + \frac{1}{2} \|W_e\|_F^2 + \lambda L \sum_{i \in [L]} \|W_i\|_F^2: f_{1-nonlin}(x;\theta) = g(x) \mbox{ for all } x \in \Omega\}\,.
\end{align*}
And we define the cost with a network with depth 2-blocks:
\begin{align*}
c^{2-nonlin}_{L,n,\lambda}(g; \Omega) = \inf_{\theta} \{ \frac{1}{2}\|W_u\|_F^2 + \frac{1}{2} \|W_e\|_F^2 + \frac{\lambda}{2} \sum_{i \in [L], j \in [2]} \|W_{i,j}\|_F^2: f_{2-nonlin}(x;\theta) = g(x) \mbox{ for all } x \in \Omega\}\,.
\end{align*}
The scaling of the cost with $L$ is taken to be the same as in the case of linear networks, so as to be in the critical scaling with nontrivial behavior when we send $L \to \infty$. We do not penalize the biases for simplicity, as this does not qualitatively change the nature of the results.

\subsection{Nonlinear rank as $\lambda \to 0$}
We study the minimum-norm solutions in the limit of taking the hyperparameter $\lambda \to 0$. We will show that the bias of these networks is towards finding a solution that minimizes the nonlinear rank, similarly to what was shown in \cite{jacot2022implicit,jacot2023bottleneck} for deep fully-connected networks without residual connections. Let us recall the notions of nonlinear rank defined in that paper.
\begin{definition}[Jacobian rank; Definition~1 of \cite{jacot2022implicit}]
The Jacobian rank of an FPLF  $g$ is given by $\rank_J(g;\Omega) = \max_{x \in \Omega} \rank(Jg(x))$, where the maximum is over points $x$ in the interior of $\Omega$ where $g$ is differentiable.
\end{definition}

\begin{definition}[Bottleneck rank; Definition~2 of \cite{jacot2022implicit}]
The bottleneck rank, $\rank_{BN}(g;\Omega)$, of an FPLF $g$ is the smallest integer $k$ such that $g = h_2 \circ h_1$ for two FPLFs $h_2 : \R^k \to \R^{n_1}$ and $h_1 : \Omega \to \R^k$.
\end{definition}

The following theorem should be compared to Theorem~1 of \cite{jacot2022implicit}, and the proof closely follows the proof in that work, except with the main difference that since we have residual connections the main bottleneck in the cost is not to represent the identity matrix. Instead, the main bottleneck is representing a large scaling of the identity matrix.

\begin{theorem}[Infinite-depth cost sandwiched by nonlinear ranks as $\lambda \to 0$]\label{thm:nonlin-k1-k2-lambda-to-0}
In the infinite-depth limit, and as the parameter $\lambda \to 0$, the cost becomes lower-bounded by the Jacobian rank:
\begin{align}\label{ineq:lower-nonlinear-sandwich-lambda-0}
\liminf_{\lambda \to 0} \liminf_{L \to \infty} \frac{c_{L,n,\lambda}^{1-nonlin}(g)}{\lambda \log(1/\lambda)^2} \geq \rank_J(x;\Omega), \quad \mbox{ and } \quad
\liminf_{\lambda \to 0} \lim_{L \to \infty} \frac{c_{L,n,\lambda}^{2-nonlin}(g)}{\lambda \log(1/\lambda)} \geq \rank_J(x;\Omega)\,.
\end{align}
On the other hand, in the infinite-depth limit, and as the parameter $\lambda \to 0$, the cost becomes upper-bounded by the bottleneck rank. For any large enough $n$ depending on $g,\Omega$
\begin{align}\label{ineq:upper-nonlinear-sandwich-lambda-0}
\limsup_{\lambda \to 0} \limsup_{L \to \infty} \frac{c_{L,n,\lambda}^{1-nonlin}(g)}{\lambda \log(1/\lambda)^2} \leq \rank_{BN}(x;\Omega), \quad \mbox{ and } \quad
\limsup_{\lambda \to 0} \lim_{L \to \infty} \frac{c_{L,n,\lambda}^{2-nonlin}(g)}{\lambda \log(1/\lambda)} \leq \rank_{BN}(x;\Omega)\,.
\end{align}
\end{theorem}
\begin{proof}

\textbf{Lower bound}. Consider the case of depth-2 blocks. Let $\theta = (W_u,W_e,b_u,b_e,\{W_{i,j}\}, \{b_{i,j}\})$ be parameters such that $f_{2-nonlin}(x;\theta) = g(x;\theta)$ for all $x \in \Omega$. At any $x \in \Omega$ where $g$ is differentiable, the Jacobian $Jg(x)$ is given by
\begin{align*}
Jg(x) = W_u(I+W_{2,L}D_L(x)W_{1,L})\dots (I+W_{2,1}D_1(x)W_{1,1}) W_e\,,
\end{align*}
where $D_1(x),\ldots,D_L(x) \in \R^{n \times n}$ are diagonal matrices with entries in $\{0,1\}$, and correspond to the derivatives of the ReLUs at the preactivations at each layer. Since $\|D_{\ell}(x)\| \leq 1$, we have $\|W_{2,\ell}D_{\ell}(x)\|_F^2 \leq \|W_{2,\ell}\|_F^2$, so it follows that
\begin{align*}
c_{L,n,\lambda}^{2-nonlin}(g) &\geq c^{2-lin}_{L,n,\lambda}(Jg(x))\,.
\end{align*}
By an analogous argument for the case of depth-1 blocks, we also have
\begin{align*}
c_{L,n,\lambda}^{1-nonlin}(g) \geq c^{1-lin}_{L,n,\lambda}(Jg(x))\,.
\end{align*}
The asymptotics in Corollary~\ref{cor:linear-networks-lambda-asymptotics} for the cost of infinite-depth 
linear networks therefore imply \eqref{ineq:lower-nonlinear-sandwich-lambda-0}.

\textbf{Upper bound}. We prove the upper bound here for residual networks with depth-2 blocks. The proof of the upper bound for depth-1 blocks is similar, and is deferred to Appendix~\ref{app:upper-depth-1-blocks}. Let $k = \rank_{BN}(g;\Omega)$, and let $g = h_2 \circ h_1$ for two FPLFs $h_2 : \R^k \to \R^{n_1}$ and $h_1 : \Omega \to \R^k$. Without loss of generality, suppose that $h_1(\Omega) \subseteq \R_{\geq 0}^k$, since we can ensure this by adding a large enough bias term. We will explicitly construct a cheap network.

\textit{First and last residual layers}. For large enough width $n$,  that there are depths $L_1,L_2$ such that $h_1$ and $h_2$ can be represented by ReLU ResNets of depth $L_1$ and $L_2$ and inner width $n$, respectively. This follows from a slight modification of Theorem 2.1 of \cite{arora2016understanding} to ResNets. In particular, for any scalar $\alpha > 0$ in our residual network, we can set  the weights of layers $f_1,\ldots,f_{L_1}$ to compute
\begin{align*}
(f_{L_1} \circ f_{L_1-1} \circ \dots \circ f_{1})(x) = [h_1(\alpha x_1,\ldots, \alpha x_{d_in})/\alpha, 0,\ldots,0] \mbox{ for all } x \in (\Omega / \alpha) \times \{0\}^{n-k}\,,
\end{align*}
and for any scalar $\beta > 0$, we can set the weights of layers $f_{L - L_2+1},\ldots,f_{L_2}$ to compute
\begin{align*}
(f_{L} \circ f_{L - 1} \circ \dots \circ f_{L-L_2+1})(x) = [h_2(\beta x_1,\ldots,\beta x_k)/\beta,0,\ldots,0] \mbox{ for all } x \in (h_1(\Omega) / \beta) \times \{0\}^{n-k}\,.
\end{align*}
Furthermore, because of the homogeneity of the ReLU activation, the sum of Frobenius norms of the weights in these layers can be made independent of $\beta,\alpha,\lambda$. Namely, $\sum_{i \in [L_1] \cup [L-L_2+1:L],j \in [2]} \|W_{i,j}\|_F^2 \leq C$ for a constant $C$ depending only $h_1,h_2$.

\textit{Intermediate residual layers}. Next, in the intermediate $\tilde{L} := L-L_1-L_2$ layers, for any scalar $\tau \geq 1$ we can choose the weights so that
\begin{align*}
(f_{L-L_2} \circ f_{L-L_2-1} \circ \dots \circ f_{L_1+1})(x) = \tau x \mbox{ for all } x \in \R^{k}_{\geq 0} \times \{0\}^{n-k}\,.
\end{align*}
Since the map only has to be the identity on the first orthant, this can be done by setting all the biases of these layers to 0. Then, for any $L_1+1 \leq \ell \leq L-L_2$, set $W_{\ell,1} = W_{\ell,2} = (\sqrt{\tau^{1/\tilde{L}}-1})\begin{bmatrix} I_k & 0 \\ 0 & 0 \end{bmatrix}$.

\textit{Embedding and unembedding layers}. Finally, the embedding and unembedding layers can be chosen to have zero biases, and weights $W_u = \sqrt{\lambda} \begin{bmatrix} I_{d_{out}} & 0 \end{bmatrix}$ and $W_e = \sqrt{\lambda} \begin{bmatrix} I_{d_{in}} \\ 0 \end{bmatrix}$.

\textit{Computing the total cost}. If we choose $\alpha = 1/\sqrt{\lambda}$, $\tau = 1/\lambda$, and $\beta = 1/\sqrt{\lambda}$, one can verify that indeed the network computes $h_2(h_1(x)) = g(x)$ for all $x \in \Omega$. This proves that the total cost is at most, in the depth-2 block case:
$c^{2-nonlin}_{L,n,\lambda}(g) \leq \lambda (C + \tilde{L}(k(1/\lambda)^{1/\tilde{L}}-1)) + \frac{1}{2}\lambda(d_{in} + d_{out})$. So 
\begin{align*}\limsup_{\lambda \to 0} \limsup_{L \to \infty} \frac{c^{2-nonlin}_{L,n,\lambda}(g)}{\lambda \log(1/\lambda)} &\leq \limsup_{\lambda \to 0} \frac{\lambda (C+d_{in}+d_{out}) + k \lambda \log(1/\lambda)}{\lambda \log(1/\lambda)} = k = \rank_{BN}(g;\Omega)\,.
\end{align*}
\end{proof}

\bibliography{bibliography}
\bibliographystyle{alpha}

\appendix

\section{Deferred details for proof of Theorem~\ref{thm:nonlin-k1-k2-lambda-to-0}}\label{app:upper-depth-1-blocks}

We provide here the upper bound for the cost for nonlinear networks with depth-1 blocks as $\lambda \to 0$. The construction of the network is similar to the construction for depth-2 blocks provided in the main text, with some technicalities because the ReLU function $\sigma(t) = \max(0,t)$ is always nonnegative, so we cannot ``zero out'' coordinates in the residual network which makes the simulation of finite-depth networks a bit more involved. Also, because the cost scales with $L$, we require the weights at all blocks to be of order $O(1/L)$ as $L \to \infty$.

Again, let $k = \rank_{BN}(g;\Omega)$, and let $g = h_2 \circ h_1$ for FPLFs $h_2 : \R^k \to \R^{d_{out}}$ and $h_1 : \Omega \to \R_{\geq 0}^k$. 

\textit{Embedding layer and first residual layers}. From \cite{arora2016understanding}, we know that there is a width $n_1$ and depth $L_1$ such that we can write
\begin{align*}
(h_1(x), 0^{n_1-k}) = \sigma(a_{L_1} + V_{L_1} \sigma(a_{L_1-1} + V_{L_1-1}\sigma(\dots \sigma(a_1 + V_1V_e x) \dots )))\,,
\end{align*}
for weights $V_1,\ldots,V_{L_1} \in \R^{n_1 \times n_1}$, biases $a_1,\ldots,a_{L_1} \in \R^{n_1}$, and embedding $V_e \in \R^{n_1 \times d_{in}}$. For any integer $m \geq 1$, we will show how to simulate this in the first $mL_1$ layers of the ResNet. Let $S_0,\ldots,S_{L_1} \subseteq [n]$ be disjoint of size $|S_i| = n$. The embedding weights are $b_e = 0$ and $[W_e]_{S_0 \times [d_{in}]} = \sqrt{\lambda}V_e$, and 0 everywhere else. For any $i \in [L_1], j \in [m]$, let $[W_{(i-1)m+j}]_{S_i \times S_{i-1}} =V_i /m$ and 0 everywhere else. Let $[b_{(i-1)m+j}]_{S_i} = a_i \sqrt{\lambda}$, and 0 everywhere else. Define
$$\tilde{h}_1 := f_{mL_1} \circ f_{mL_1 - 1} \circ \dots \circ f_1 \circ f_e$$ and let $T \subseteq S_L$ be the first $k$ coordinates of $S_L$. This has the property that, for all $x \in \Omega$, we have $$[\tilde{h}_1(x)]_{T} = \sqrt{\lambda} h_1(x).$$
And for all $i \not\in S_0 \cup \dots \cup S_L$,
$$[\tilde{h}_1(x)]_{i} = 0.$$
Furthermore the total cost of these first $mL_1$ layers and embedding layer is at most
\begin{align*}
\frac{1}{2} \|W_e\|_F^2 + \lambda L 
\sum_{i=1}^{mL_1} \|W_i\|_F^2 \leq \frac{\lambda}{2} \|V_e\|_F^2 + \lambda L \sum_{i=1}^{mL_1} \|V_{\ceil{i/m}}\|_F^2 / m^2 \leq \lambda C_1(1 + L / m)\,,
\end{align*}
for a constant $C_1$ depending on $h_1$.

\textit{Intermediate residual layers}. Next, in the intermediate $L_{int}$ layers, for any scalar $\tau \geq 1$ we can choose the weights so that
\begin{align*}
\tilde{h}_{int} := f_{L_1+L_{int}} \circ f_{L_1+L_{int}-1} \circ \dots \circ f_{L_1+1}
\end{align*}
satisfies
\begin{align*}
[\tilde{h}_{int}(x)]_i = \begin{cases} \tau [x]_i, & \mbox{ if } i \in T \\
[x]_i, & \mbox{ if } i \not\in T \end{cases} \mbox{ for all } x \in \R^n\,,
\end{align*}
Since the map only has to be the identity on $k$ nonnegative coordinates, this can be done by setting all the biases of these layers to 0, and $[W_{\ell}]_{T \times T} = (\tau^{1/\tilde{L}}-1) I_k$, and 0 everywhere else.

\textit{Last residual layers and unembedding layer}. Finally, by a similar argument to the embedding and first layers, there is $L_2$ and $n_2$ such that we can simulate $h_2$ in the final $mL_2$ and unembedding layers, i.e.
\begin{align*}
\tilde{h}_2 := f_u \circ f_{L} \circ f_{L-1} \circ \dots \circ f_{L-mL_2+1}
\end{align*}
satisfies, for all $x$ such that $x_{[T]} \subseteq h_1(\Omega)/\sqrt{\lambda}$ and such that $x_i = 0$ if $i \not\in S_0 \cup \dots \cup S_L$,
\begin{align*}\tilde{h}_2(x) = h_2(\sqrt{\lambda}x_{[T]})\end{align*}
and we have cost
\begin{align*}
\frac{1}{2} \|W_u\|_F^2 + \lambda L \sum_{i=L-mL_2+1}^{L} \|W_i\|_F^2 \leq \lambda C_2 (1 + L/m)\,.
\end{align*}

\textit{Computing the total cost}. Letting $L_{int} = L - mL_1 - mL_2$ and $\tau = 1/\lambda$, we see that $f_{1-lin}(x;\theta) = \tilde{h}_2(\tilde{h}_{int}(\tilde{h}_1(x))) = h_2(h_1(x)) = g(x)$ for all $x \in \Omega$. So the total cost in the depth-1 case for large enough $n$ is at most
\begin{align*}c^{1-nonlin}_{L,n,\lambda}(g) \leq \min_{m \mbox{ such that } mL_1 + mL_2 < L} \lambda(C_1+C_2)(1+L/m) + \lambda L(L_{int}k((1/\lambda)^{1/L_{int}}-1)^2).\end{align*} 
If we choose $m = \floor{L / 
 \log(1/\lambda)}$ for small enough $\lambda$, this shows 
\begin{align*}\limsup_{\lambda \to 0} \limsup_{L \to \infty} \frac{c^{1-nonlin}_{L,n,\lambda}(g)}{\lambda \log(1/\lambda)^2} &\leq \limsup_{\lambda \to 0} \frac{\lambda (C_1+C_2)(2+\log(1/\lambda)) + k \lambda \log(1/\lambda)^2/(1-(L_1+L_2)/\log(1/\lambda))}{\lambda \log(1/\lambda)^2} \\
&= k = \rank_{BN}(g;\Omega)\,.
\end{align*}

\end{document}